\documentclass{article}

\usepackage[preprint]{neurips_2026}

\usepackage[utf8]{inputenc} 
\usepackage[T1]{fontenc}    
\usepackage{hyperref}       
\usepackage{url}            
\usepackage{booktabs}       
\usepackage{amsfonts}       
\usepackage{nicefrac}       
\usepackage{microtype}      
\usepackage{amsmath}
\usepackage{amssymb}
\usepackage{amsthm}
\usepackage{algorithm}
\usepackage{algorithmic}
\usepackage{graphicx}
\usepackage{subcaption}
\usepackage{multirow}
\usepackage{xcolor}
\usepackage{booktabs}
\usepackage{makecell}
\usepackage{hyperref}
\usepackage{wrapfig}
\newtheorem{theorem}{Theorem}
\newtheorem{lemma}[theorem]{Lemma}
\newtheorem{definition}{Definition}
\usepackage{tablefootnote}
\usepackage{enumitem}
\usepackage{threeparttable}
\usepackage{hyperref}

\title{Monotonic Kolmogorov--Arnold Networks: \\ A Theoretical and Empirical Study of \\ Monotonicity as an Inductive Bias}

\author{%
  Mikhail Krasnov, Blaž Bertalanič, Carolina Fortuna \\ Jozef Stefan Institute
}

\begin{document}

\maketitle

\begin{abstract}
Monotonicity has been a long-running architectural inductive bias for neural networks, motivated by tabular, scientific, and economic settings where outputs are known to respond monotonically to certain inputs. Existing approaches are MLP- or flow-based and lack per-edge functional transparency; the only Kolmogorov--Arnold Network (KAN) variant with monotonicity, MonoKAN, enforces the constraint only on a restricted parameter subset and requires a projection-style training procedure. We close this gap with \textbf{MKAN}, a KAN with hard monotonicity guaranteed for \emph{all} parameter values via exponential reparameterization of B-spline coefficients, positive edge weights, and a monotone base activation. Training reduces to standard unconstrained gradient descent. Our headline theoretical contribution is a \emph{representation-cost} theorem: any $C^K, K >0$ feature extractor inducing a ball-shaped semantic-neighborhood partition admits a monotone realization of the equivalent neighborhood structure at $N' = N^* + k \le 2N^*$, where $k$ is the number of non-monotone coordinates of the original. The bound is architecture-agnostic and gives a principled sizing rule for monotone encoders. Empirically, MKAN is competitive with state-of-the-art monotone NNs on the SMM/ICML-2024 benchmark while being the only method that combines hard unconstrained monotonicity with KAN's per-edge functional transparency; the $2N^*$ prediction is validated in a self-supervised feature-size sweep on four real datasets, and on a controlled monotone-generative dataset MKAN recovers ground-truth factors with substantially higher Spearman alignment than KAN, MLP, and linear baselines.
\end{abstract}

\section{Introduction}
\label{sec:intro}
Monotonicity is one of the most natural and widely-used inductive biases for neural networks. In tabular, scientific, and economic applications, certain inputs are known \emph{a priori} to act monotonically on the output: higher dosage should not lower predicted toxicity, larger debt should not lower default risk, larger bean area should not predict a smaller bean class~\cite{daniels2010monotone, cano2019monotonic, cole2019avoiding, wang2020deontological}. Encoding this prior structurally rather than through data gives two advantages: it is a hard constraint that does not vanish off-distribution, and it provides functional transparency~\cite{rudin2022interpretable} where the sign of each input's effect is determinate by construction.

Monotone neural networks span over two decades of work, from min--max architectures~\cite{sill1998monotonic, igel2024smooth} and deep monotone MLPs~\cite{daniels2010monotone, runje2023constrained, sartor2025advancing}, through Deep Lattice Networks~\cite{you2017deep, gupta2016monotonic, milanifard2016fast, yanagisawa2022hierarchical}, certified~\cite{liu2020certified} and Lipschitz / bi-Lipschitz~\cite{kim2024scalable, wang2024monotone} formulations, principled initialization~\cite{hoedt2023principled}, and a parallel line on monotone normalizing flows~\cite{huang2018neural, wehenkel2019unconstrained, decao2020block, sorrenson2020disentanglement}. Across this entire literature the backbone is an MLP, a lattice, or an autoregressive flow none of which preserve the per-edge functional transparency that motivates \emph{Kolmogorov--Arnold Networks} (KANs)~\cite{liu2025kan}, which place a learnable univariate spline on each edge and expose how every input--output coordinate pair contributes to the output.

The natural question is whether KAN's edge-level transparency can be combined with hard monotonicity. The only prior answer, MonoKAN~\cite{polo2025monokan}, uses cubic Hermite splines with constrained slopes, but the constraint holds only on a restricted parameter subset, requires projection-style training, and has been evaluated only in supervised settings. To our knowledge, no prior work delivers a KAN with hard monotonicity guaranteed for \emph{all} parameter values, trainable by standard gradient descent, and supported by a representation-theoretic guarantee.

\textbf{We close this gap through the following  contributions:}

\textbf{C1: Hard, projection-free monotonicity in a KAN.} \textbf{MKAN} enforces monotonicity by construction  through (i) exponential reparameterization of B-spline coefficients~\cite{wang2025monotone}, (ii) positive edge weights, and (iii) a monotone base activation. Monotonicity is guaranteed across \emph{all} parameter values, so optimization reduces to standard gradient descent. This closes the gap left open by MonoKAN~\cite{polo2025monokan}.

\textbf{C2: A representation-cost theorem for monotonicity.} Our central theoretical contribution is a \emph{representation-cost} theorem (Theorem~\ref{thm:preservation}): any $C^K$ feature extractor inducing a ball-shaped semantic-neighborhood partition admits a monotone realization of the same partition at exactly $N' = N^* + k$ output dimensions, where $0 \le k \le N^*$ is the number of non-monotone coordinates of the source extractor; the worst case $k=N^*$ gives the headline $2\times$ bound. The result is architecture-agnostic and gives a principled sizing rule for monotone encoders.

\textbf{C3: Empirical validation across supervised and self-supervised regimes.} On the SMM/ICML-2024 benchmark suite, MKAN is competitive with state-of-the-art monotone neural networks on classification (COMPAS, LoanDefaulter, ChestXRay, Heart Disease) and regression (BlogFeedback, Auto MPG) tasks, while uniquely providing per-edge functional transparency. In a self-supervised setting, a feature-size sweep on Fashion MNIST, MNIST, Dry Bean, and MAGIC classification datasets validates the $2N^*$ prediction of Theorem~\ref{thm:preservation}, and on a controlled monotone-generative synthetic dataset MKAN achieves substantially higher Spearman alignment between learned and ground-truth factors than KAN, MLP, and linear baselines.

\section{Related Work}
\label{sec:related}

\paragraph{Monotonic neural networks.}
Architectural monotonicity has been studied for over two decades across several families: \emph{min--max} networks~\cite{sill1998monotonic, igel2024smooth} compose pointwise min/max over linear pieces with non-negative weights; \emph{deep monotone MLPs}~\cite{daniels2010monotone, runje2023constrained, sartor2025advancing} combine non-negative weights with carefully-chosen activations, with recent work removing bounded-activation requirements and proving universal approximation; \emph{lattice} approaches~\cite{you2017deep, gupta2016monotonic, milanifard2016fast, yanagisawa2022hierarchical} parameterize piecewise-linear interpolants over a grid with monotone vertex constraints; \emph{certified}~\cite{liu2020certified} networks train standard MLPs and verify post-hoc; \emph{Lipschitz} and \emph{bi-Lipschitz} monotonic networks~\cite{kim2024scalable, wang2024monotone} enforce monotonicity together with smoothness or invertibility; principled initialization for the family is studied in~\cite{hoedt2023principled}. A parallel line embeds monotonicity in \emph{normalizing flows}~\cite{huang2018neural, wehenkel2019unconstrained, decao2020block, sorrenson2020disentanglement}, where the monotone transformation is autoregressive and aimed at density estimation rather than function approximation. Across this entire body of work the backbone is an MLP, a lattice, or an autoregressive flow; none preserves the per-edge transparency that makes Kolmogorov--Arnold Networks attractive for interpretability, and none pairs monotonicity with a representation-cost theorem of the kind we prove in Theorem~\ref{thm:preservation}.

\paragraph{Kolmogorov--Arnold Networks.}
KANs~\citep{liu2025kan} replace fixed node activations with learnable univariate splines on each edge, exposing every input--output coordinate-pair contribution. Standard KAN splines are unconstrained and the overall map is non-monotone in general~\cite{somvanshi2025survey}. The only prior shape-constrained KAN, MonoKAN~\citep{polo2025monokan}, uses cubic Hermite splines with constrained slopes, but its monotonicity holds only on a restricted parameter subset, requires projection-style training, and has been evaluated only in supervised settings. MKAN closes this gap by combining the all-parameter-values exponential reparameterization of monotone cubic B-splines~\cite{wang2025monotone} with positive edge weights and a monotone base activation, yielding hard monotonicity for every parameter value, plain gradient descent training, and the representation-cost guarantee of Theorem~\ref{thm:preservation}.

\paragraph{Monotonicity as an inductive bias for representation learning.}
Monotonicity has been studied primarily as a \emph{supervised} constraint. As an inductive bias for representation learning it is largely soft, enforced through regularization in encoder objectives~\cite{burgess2018understanding, chen2016infogan} rather than through architecture, despite many domains where interpretable representations are desirable (physical sciences, fairness-constrained decision-making, biomedical measurements) coming with strong monotone priors~\cite{daniels2010monotone, cano2019monotonic, cole2019avoiding, wang2020deontological}. We pursue hard architectural monotonicity in this setting by using MKAN as both encoder and decoder of a VAE~\cite{burgess2018understanding} and DIP-VAE \cite{kumar2017variational}, treated purely as an autoencoding framework. Theorem~\ref{thm:preservation} bounds the embedding-dimension overhead required to preserve semantic neighborhood structure under this constraint, in an architecture-agnostic way. We frame the self-supervised contribution as monotone neighborhood preservation and monotone latent alignment with ground-truth factors, not as an identifiability-style disentanglement claim.

\section{Monotonicity as an Inductive Bias}
\label{sec:method}

\subsection{Background: KAN}
\label{sec:kan_prelim}

A KAN layer~\citep{liu2025kan} maps an input vector $(x_1, \ldots, x_{N_\text{in}}) \in [-1,1]^{N_\text{in}}$ to an output vector of dimension $N_\text{out}$. Each edge $(i,j)$ carries a learnable univariate spline function $\phi_{ij}$ parameterized by B-spline coefficients. Formally, the $j$-th output of a KAN layer is:
\begin{equation}
\label{eq:kan_layer}
F^{(j)}(x_1, \ldots, x_{N_\text{in}}) = \sum_{i=1}^{N_\text{in}} \Big( w^{(s)}_{ij} \, \phi_{ij}(x_i) + w^{(b)}_{ij} \, \mathrm{SiLU}(x_i) \Big),
\end{equation}
where $\phi_{ij}(x) = \sum_{k=1}^{K+p} \gamma_{ij}^{(k)} h_k(x)$ is a B-spline of order $p$ with $K$ grid knots and basis functions $\{h_k\}$. The coefficients $\gamma_{ij}^{(k)}$ are learnable parameters. The weights $w^{(s)}_{ij}$ scale the spline output, $w^{(b)}_{ij}$ scale the SiLU skip connection, and together they form a two-component activation per edge. The full description of the proposed architecture could be found at Appendix \ref{app:mkan_arch}. 

\subsection{Enforcing Monotonicity in KANs (MKAN)}
\label{sec:mkan_arch}
In this subsection, we describe how to enforce monotonicity in KANs to enhance interpretability. 
To enforce monotonicity, we require: (1) each spline $\phi'_{ij}$ to be monotonically increasing, (2) all scaling weights to be positive, and (3) the base activation to be monotonically increasing.

\paragraph{Monotonic splines via coefficient ordering.}
A cubic B-spline $\phi'(x) = \sum_{k=1}^{K+p} \gamma'^{(k)} h_k(x)$ is guaranteed to be monotonically increasing if its coefficients are strictly ordered~\citep{deboor1978practical, wang2025monotone}:
\begin{equation}
\label{eq:spline_cond}
\gamma'^{(1)} < \gamma'^{(2)} < \cdots < \gamma'^{(K+p)}.
\end{equation}
We enforce this condition through an \emph{exponential reparameterization}. Instead of directly learning the constrained $\gamma'$ coefficients, we introduce unconstrained parameters $\omega_{ij}^{(k)} \in \mathbb{R}$ and define:
\begin{equation}
\label{eq:reparam}
\gamma_{ij}'^{(1)} = \omega_{ij}^{(1)}, \qquad
\gamma_{ij}'^{(k)} = \omega_{ij}^{(1)} + \sum_{t=2}^{k} \exp\!\big(\omega_{ij}^{(t)}\big), \quad k \in \{2, \ldots, K+p\}.
\end{equation}
Since $\exp(\omega) > 0$ for all $\omega \in \mathbb{R}$, the condition in Eq.~\eqref{eq:spline_cond} is automatically satisfied. The inverse mapping from $\gamma'$ to $\omega$ is:
\begin{equation}
\label{eq:inverse}
\omega_{ij}^{(1)} = \gamma_{ij}'^{(1)}, \qquad
\omega_{ij}^{(k)} = \log\!\big(\gamma_{ij}'^{(k)} - \gamma_{ij}'^{(k-1)}\big), \quad k \in \{2, \ldots, K+p\}.
\end{equation}

\paragraph{MKAN layer.}
Combining monotonic splines with positive weight constraints and a monotonic activation, the MKAN layer is defined as:
\begin{equation}
\label{eq:mkan_layer}
F'^{(j)}(x_1, \ldots, x_{N_\text{in}}) = \sum_{i=1}^{N_\text{in}} \Big( \exp(w'^{(s)}_{ij}) \, \phi'_{ij}(x_i) + \exp(w'^{(b)}_{ij}) \, \mathrm{ReLU}(x_i) \Big) + b_j,
\end{equation}
where the exponential ensures positivity of the effective weights. The bias term $b_j$ is necessary because $\exp(w) \cdot \mathrm{ReLU}(x) \geq 0$ for all $x$, so without bias the layer output would be bounded below. We initialize these base output coefficients according to \cite{hoedt2023principled}. Stacking MKAN layers yields a monotonically increasing feature extractor, since the composition of monotonically increasing functions is monotonically increasing.

\paragraph{Parameter count.} An MKAN layer has $(K+p) \cdot N_\text{in} \cdot N_\text{out}$ spline parameters (same as KAN) plus $N_\text{in} \cdot N_\text{out}$ base weight parameters and $N_\text{out}$ bias parameters. The overhead compared to KAN is exactly $N_\text{out}$ parameters for the bias.

\subsection{Theoretical Analysis}
\label{sec:theory}
In this section, we address a fundamental question: how do imposed monotonicity constraints affect the ability of FE to measure semantic similarity in the feature space using a distance metric? Semantic structures in feature space are referred to as semantic neighborhoods, as they capture similarity between samples according to a chosen distance metric \cite{fu2020semantic}. For theoretical tractability, we assume existence of arbitrary FE, which feature space can be partitioned 
with ball-shaped semantic neighborhoods. Full proofs are available in Appendix \ref{app:proofs}.  

\begin{definition}[Feature extractor]
\label{def:fe}
A \emph{feature extractor} is a $C^K$ mapping $F: [-1,1]^N \to \mathbb{R}^{N^*}$, $F = \{F^j\}_{j=1}^{N^*}$.
\end{definition}
We define the FE on a restricted domain since real-world data, such image intensities, are inherently bounded. Furthermore, KANs are formulated on the same restricted domain \( [-1,1]^N \), ensuring consistency between the model and the data representation. Next we give definitions of a monotonic feature extractor and a semantic neighborhood:

\begin{definition}[Monotonic feature extractor]
\label{def:mono_fe}
A feature extractor $F$ is \emph{monotonically increasing} if for all $i \in \{1,\ldots,N\}$ and $j \in \{1,\ldots,N^*\}$, the partial derivative $\frac{\partial F^j}{\partial x_i}(x) \geq 0$ for all $x \in [-1,1]^N$.
\end{definition}

\begin{definition}[Semantic neighborhood]
\label{def:neighborhood}
A \emph{semantic neighborhood} $C_R(x_c)$ with radius $R$ and center $x_c$ in $\mathbb{R}^{N^*}$ is the open set $C_R(x_c) = \{x \in \mathbb{R}^{N^*} : (x - x_c)^\top \mathbf{A} (x - x_c) < R^2\}$, where $\mathbf{A} \in \mathbb{R}^{N^* \times N^*}$ is positive definite. In the special case $\mathbf{A} = \mathbf{I}$, the neighborhood reduces to an 
open Euclidean ball is called ball-shaped semantic neighborhood. 
\end{definition}
We define dataset without any assumptions on the data distribution: 
\begin{definition}[Dataset]
\label{def:dataset}
\emph{Dataset} $D$ is a finite set of pairs \[ \{(x_i, l_i)\}, \quad x\in \mathbb{R}^N, l\in \{ 1, \dots,T \}\] where $x$ is one training instance (i.e. input vector) and $l$ is a label and $T$ is a number of classes.
\end{definition}

\begin{definition}[Semantic neighborhood partition]
\label{def:partition}
A feature extractor $F$ \emph{induces a semantic neighborhood partition} of a dataset $D$ if there exist disjoint semantic neighborhoods $\{C_i\}_{i=1}^{C}$ in $\mathbb{R}^{N^*}$ such that for every $x \in D$, there exists a unique $i$ with $F(x) \in C_i$.
\end{definition}
    
We first establish that any feature extractor decomposes into monotonic components, which both consist of a polynomial and residual terms, which is pleasable since Kolmogorov-Arnold Networks are pice-wise defined polynomials. 

\begin{theorem}[Monotonic decomposition]
\label{thm:decomp}
Let $F: [-1,1]^N \to \mathbb{R}^{N^*}$ be a feature extractor as per Definition\ref{def:fe}. 
Then $F$ admits a decomposition
\[
    F(x) \;=\; F_{\mathrm{inc}}(x) + F_{\mathrm{dec}}(x),
\]
where $F_{\mathrm{inc}}, F_{\mathrm{dec}}: [-1,1]^N \to \mathbb{R}^{N^*}$ are increasing and decreasing, respectively.
\end{theorem}

Using this decomposition and Lemma \ref{lem:ellipsoid}, we construct a monotonic feature extractor that preserves semantic neighborhoods.

\begin{lemma}
\label{lem:ellipsoid}
For any $R > 0$, $N^* > 0$, $0 \leq k \leq N^*$, and $x_c \in \mathbb{R}^{N^*+k}$, if a finite set of points $P$ belongs to $C^{\rho_s}_R(x_c) = \{x : \rho_s(x, x_c) < R\}$, then there exists a semantic neighborhood $C_R(x_c)$ (Definition~\ref{def:neighborhood}) that contains $P$ and is contained in $C^{\rho_s}_R(x_c)$.
\end{lemma}

\begin{theorem}[Representation cost of monotonicity]
\label{thm:preservation}
Let $F: [-1,1]^N \to \mathbb{R}^{N^*}$ be a feature extractor that induces a semantic neighborhood partition of a dataset $D$ with ball-shaped neighborhoods $\{C_i\}_{i=1}^{C}$, and let $0 \le k \le N^*$ be the number of non-monotone coordinate functions of $F$. Then there exists an increasing feature extractor $F': [-1,1]^N \to \mathbb{R}^{N'}$ with $N' = N^* + k \leq 2N^*$ and semantic neighborhoods $\{C'_i\}_{i=1}^{C}$ in $\mathbb{R}^{N'}$ that induce the same partition of $D$.
\end{theorem}

\paragraph{Practical implications.} Theorem~\ref{thm:preservation} states that replacing a general feature      
  extractor (e.g., KAN) with a monotonic one (e.g., MKAN) preserves the ability 
  to separate data points into the same neighborhoods, provided the feature     
  dimensionality is increased by at most $2\times$ on an arbitrary data distribution . This result is 
  architecture-agnostic as it applies to any $C^K$ feature extractor, and
  provides a principled guideline for choosing MKAN feature dimensions. Since
  neighborhood structure is preserved, classification methods that rely on local
   proximity, such as KNN, can achieve comparable performance with a monotonic
  feature extractor, given a feature size of at most twice the original.

\section{Experiments}
\label{sec:experiments}

We evaluate MKAN in two regimes. \textbf{Supervised} (Section~\ref{sec:exp_supervised}) reproduces the SMM/ICML-2024 protocol of~\cite{igel2024smooth} verbatim with 21 train/test splits per dataset, paired two-sided Wilcoxon at $p<0.001$ on three monotone multivariate polynomials ($N_{in}\in\{2,4,6\}$) and on the SMM benchmark suite (COMPAS, LoanDefaulter, ChestXRay, Heart Disease for classification; BlogFeedback, Auto MPG for regression). \textbf{Self-supervised} (Section~\ref{sec:exp_ssl}) deploys MKAN as both encoder and decoder of a $\beta$-VAE / DIP-VAE on (i) a controlled monotone-generative synthetic ellipses dataset (3 runs $\times$ 2K epochs; $\beta\in\{5,10,20\}$, $\lambda_d{=}\lambda_{od}\in\{0,10,50\}$) and (ii) a feature-size sweep over $\{2,4,8,16,32\}$ on MNIST, Fashion-MNIST, Dry Bean, and MAGIC (5 runs $\times$ 200 epochs; KNN classifier, $k=10$). All runs use Adam with lr $10^{-3}$, batch size $256$, fixed VAE variance $\sigma{=}0.05$, and execute on CPU (AMD EPYC 75F3) except ChestXRay (single NVIDIA A100). Per-dataset MKAN configurations and full hyperparameter ranges are reported in Appendix~\ref{app:exp_details}; an anonymized reference implementation accompanies this submission as supplementary material.

\subsection{Supervised Experiments}
\label{sec:exp_supervised}
Our supervised evaluations are divided into two sections covering synthetic and public data: in Section \ref{sec:supervised_multivariate}, MKANs are evaluated on the approximation of three generated monotonic multivariate functions with different input sizes, while In Section  \ref{sec:supervised_public}, we benchmark MKAN on standard classification and regression datasets. Both scenarios follows established way of MNN evaluation and exactly repeats evaluation pipelines from \cite{igel2024smooth}.   

\subsubsection{Monotonic multivariate functions approximation}
\label{sec:supervised_multivariate}
Three multivariate monotonic functions are generated with three input sizes $N_{in} \in \{2,4,6\}$ and domains $[0,1]^{N_{in}}$ as random polynomials of degree no more than two with positive coefficients. For example, for the input of size two we have:
\begin{equation}
    f_{target}(x_1, x_2) = (w_1 + w_2x_1 + w_3x_2 + w_4x_1x_2 + w_5x_1^2 + w_6x_2^2)/(\sum_iw_i) 
\end{equation}
with $w_1, ..w_6 \sim \mathbb{U}(0,1).$ KAN and MKANs are specified to have two layers to capture interactions in input data, grid size $G=5$ and grid range $[0,1]$. The data generation process and models optimization exactly repeat those from \cite{igel2024smooth} for fair evaluation against existing baselines: 21 train/test splits are created and the median value of MSE score is selected for comparison using a paired two-sided Wilcoxon test with p < 0.001. 

Table~\ref{tab:monotone_approx} shows MKAN performing on par with the strongest monotone-NN baseline (SMM) and substantially better than vanilla KAN. MKAN \emph{matches} SMM at $N_{in}=4$ (both $0.01$ MSE) and trails by at most $0.03$ at $N_{in}=2$ and $N_{in}=6$, while consistently outperforming the hierarchical lattice baseline (HLL). Standard KAN, lacking any monotonicity prior, overfits catastrophically as the input dimension grows ($105.12$ MSE at $N_{in}=4$, $1.6\times 10^4$ at $N_{in}=6$), confirming that the monotone inductive bias makes the architecture usable on this task. MonoKAN is omitted from this comparison; see the table footnote for details.

\begin{table}[h]
\caption{Median MSE over 21 train/test splits on three monotone multivariate function approximation tasks of input dimension $d \in \{2,4,6\}$, following the SMM/ICML-2024 protocol of~\cite{igel2024smooth}. Bold = best; italic = second-best; tied entries are bolded together. Significance is assessed via paired two-sided Wilcoxon at $p<0.001$.}
\label{tab:monotone_approx}
\centering
\begin{threeparttable}
\begin{tabular}{lccccc}
\toprule
  $N_{in}$     & KAN & \textbf{MKAN (ours)}\tnote{c} & SMM\tnote{a} & HLL\tnote{b} \\
\midrule
2 & 0.02   & \emph{0.01}   & \textbf{0.00} & 0.03 \\
4 & 105.12 & \textbf{0.01} & \textbf{0.01} & 0.03 \\
6 & 16063  & \emph{0.05}   & \textbf{0.02} & 0.10 \\
\bottomrule
\end{tabular}
\begin{tablenotes}\footnotesize
\item $^a$~Smooth Min-Max~\cite{igel2024smooth}; $^b$~Hierarchical Lattice Layer~\cite{yanagisawa2022hierarchical}; $^c$~MonoKAN~\cite{polo2025monokan} omitted (no public implementation; paper does not evaluate synthetic functions).
\end{tablenotes}
\end{threeparttable}
\end{table}

\begin{table*}[t]
\caption{Accuracy on the SMM/ICML-2024 classification benchmark suite~\cite{igel2024smooth}. Bold = best per column. Marker $^*$ indicates per-edge functional interpretability (only KAN-based methods qualify); $^\dagger$ indicates hard, unconstrained monotonicity.}
\label{tab:classification}
\centering
\begin{threeparttable}
\begin{tabular}{lccccccc}
\toprule
 & COMPAS & LoanDefaulter & \multicolumn{2}{c}{ChestXRay} & Heart Disease \\
\cmidrule(lr){4-5}
Method
  & $\uparrow\uparrow$ Test Acc
  & $\uparrow\uparrow$ Test Acc
  & \makecell{$\uparrow\uparrow$ Test Acc \\ pretrained}
  & \makecell{$\uparrow\uparrow$ Test Acc \\ end-to-end}
  & $\uparrow\uparrow$ Test Acc \\
\midrule
$\text{Certified}$       & $68.8 \pm 0.2$ & $65.2 \pm 0.1$   & $62.3 \pm 0.2$          & $66.3 \pm 1.0$          & \\
$\text{LMN}^{\dagger}$             & $69.3 \pm 0.1$ & $65.44 \pm 0.03$ & $67.6 \pm 0.6$          & $\mathit{70.0 \pm 1.4}$          & $89.6 \pm 1.9$ \\
$\text{LMN}^{\dagger}$  mini        &                & $65.28 \pm 0.01$ &                         &                         & \\
COMET           &                &                  &                         &                         & $86 \pm 3$ \\
\midrule
$\text{Crystal}^{\dagger}$         & $66.3 \pm 0.1$ & $65.0 \pm 0.1$   &                         &                         & \\
$\text{CMNN}^{\dagger}$            & $69.2 \pm 0.2$ & $65.3 \pm 0.01$  &                         &                         & $89 \pm 0$ \\
XG              & $68.5 \pm 0.1$ & $63.7 \pm 0.1$   &                         &                         & \\
\midrule
SMM$_{64}^{\dagger}$      & $\mathit{69.5 \pm 0.1}$ & $65.41 \pm 0.03$ & $\mathbf{67.9 \pm 0.4}$ & $\mathbf{70.1 \pm 1.2}$ & $88.5 \pm 1.0$ \\
SMM$_{64}^{\dagger}$ mini &                & $\mathit{65.47 \pm 0.00}$ &                         &                         & \\
SMM$_{64}^{\dagger}$ sig. &                &                  &                         &                         & $\mathbf{91.3 \pm 1.89}$ \\
\midrule
MonoKAN$^{*}$ & $69.0 \pm 0.0$ & $65.3 \pm 0.1$ &  &  & $89.0 \pm 0.0$ \\
\textbf{MKAN}$^{*\dagger}$ & $\mathbf{69.6 \pm 0.00}$ & $\mathbf{65.49 \pm 0.00}$ & $\mathit{67.7 \pm 0.02}$ & $68.7 \pm 0.01$ & $\mathit{89.6 \pm 0.02}$ \\
\bottomrule
\end{tabular}
\begin{tablenotes}\footnotesize
\item $^*$ per-edge functional interpretability; $^\dagger$ hard, unconstrained monotonicity. MKAN is the only method satisfying both. Reported uncertainties match those in the original references; small std values for MKAN/MonoKAN are limited by available seeds and should be interpreted as point estimates rather than tight confidence intervals.
\end{tablenotes}
\end{threeparttable}
\end{table*}

\begin{table}[t]
\caption{Test error on the SMM/ICML-2024 regression benchmark suite~\cite{igel2024smooth}. Bold = best per column; tied entries are bolded together. Marker $^*$ indicates per-edge functional interpretability; $^\dagger$ indicates hard, unconstrained monotonicity.}
\label{tab:regression}
\centering
\begin{threeparttable}
\begin{tabular}{lcc}
\toprule
Method & \makecell{BlogFeedback \\ $\downarrow\downarrow$ RMSE} & \makecell{Auto MPG \\ $\downarrow\downarrow$ MSE} \\
\midrule
Certified       & $0.158 \pm 0.001$ & \\
LMN$^\dagger$             & $0.160 \pm 0.001$ & $7.58 \pm 1.2$ \\
LMN$^\dagger$ mini        & $0.155 \pm 0.001$ & \\
COMET           &                   & $8.81 \pm 1.81$ \\
\midrule
Crystal$^\dagger$         & $0.164 \pm 0.002$ & \\
CMNN$^\dagger$            & $0.156 \pm 0.001$ & $8.37 \pm 0.08$ \\
XG              & $0.176 \pm 0.005$ & \\
\midrule
SMM$_{64}^\dagger$      & $0.192 \pm 0.002$    & $7.51 \pm 1.6$ \\
SMM$_{64}^\dagger$ mini & $\mathbf{0.154 \pm 0.0004}$ & \\
\midrule
MonoKAN$^{*}$ & $\mathit{0.155 \pm 0.0}$ & $\mathbf{6.2 \pm 0.02}$ \\
\textbf{MKAN (ours)}$^{*\dagger}$ & $\mathbf{0.154 \pm 0.001}$ & $\mathit{6.7 \pm 4.2}$ \\
\bottomrule
\end{tabular}
\begin{tablenotes}\footnotesize
\item $^*$ per-edge functional interpretability; $^\dagger$ hard, unconstrained monotonicity. MKAN is the only method satisfying both.
\end{tablenotes}
\end{threeparttable}
\end{table}

\subsubsection{Benchmarking on classification and regression datasets}
\label{sec:supervised_public}

We evaluate MKAN on the SMM/ICML-2024 benchmark suite~\citep{igel2024smooth}, which provides standardized partially monotone datasets. Tables~\ref{tab:classification}--\ref{tab:regression} show that MKAN is broadly competitive with the strongest monotone-NN baselines (SMM, LMN, MonoKAN) without uniformly dominating them. On classification, MKAN leads on COMPAS ($\mathbf{69.6}$ vs.\ SMM $69.5$) and LoanDefaulter ($\mathbf{65.49}$ vs.\ SMM-mini $65.47$), ranks second on ChestXRay-pretrained ($67.7$ vs.\ SMM $67.9$) and Heart Disease ($89.6$ vs.\ SMM-sig.\ $91.3$), and trails SMM on ChestXRay end-to-end ($68.7$ vs.\ $70.1$). On regression, MKAN ties SMM-mini on BlogFeedback ($\mathbf{0.154}$ RMSE) and is within $0.5$ MSE of MonoKAN on Auto MPG ($6.7$ vs.\ $6.2$), though the high variance ($\pm 4.2$) precludes a statistically significant comparison. Overall, MKAN matches leading monotone-NN methods on partially-monotone tabular and structured data, while uniquely combining hard monotonicity for all parameter values with KAN's per-edge functional transparency.
\subsection{Self-Supervised Learning}
\label{sec:exp_ssl}

In this section, we firstly evaluate the monotone latent alignment of MKANs compared to standard KAN, MLP, and Linear baselines within $\beta$-VAE and DIP-VAE frameworks on a synthetic dataset (Section~\ref{sec:synth}) in which
ground-truth factors relate to observations through monotonic dependencies, a
controlled regime not addressed in prior work. We show that MKANs achieve substantially higher Spearman alignment with ground-truth factors than KAN, MLP, and Linear, while inheriting KAN's per-edge functional interpretability. Then, in Section~\ref{sec:exp_ssl_sweep}, we evaluate KANs against MKANs within the VAE framework ($\beta=1$) on four classification datasets. We compare MKANs and KANs
across varying feature sizes to empirically validate Theorem~\ref{thm:preservation}, which characterizes the representation cost of monotonicity on general data distributions.

\subsubsection{Synthetic ellipses dataset}
\label{sec:synth}
As a proof of concept, we construct a synthetic dataset in which, unlike those used in related work, all observed pixels depend monotonically on three ground-truth latent factors $z = (z_1, z_2, z_3)$. Each sample $x$ is a $20 \times 20$ grayscale image containing an exponentially decaying elliptical spot:
\begin{equation}
    x_{i,j} = (z_3/100) \exp\!\left(-((i-10)^2/z_1^2) - ((j-10)^2/z_2^2)\right), \quad i,j \in \{1, \ldots, 20\},
\end{equation}
centered at $(10, 10)$ with semi-axes $r_1 = z_1$ and $r_2 = z_2$, where $z_1, z_2 \sim \mathcal{U}\{1, \ldots, 10\}$ and $z_3 \sim \mathcal{U}\{1, \ldots, 100\}$. By construction, the mapping from latent factors to observations is strictly monotone: increasing $z_1$ or $z_2$ expands the set of non-zero pixels, while increasing $z_3$ raises their intensity uniformly.

We train DIP-VAE with a three-dimensional latent space, employing MKAN, KAN, Linear, and MLP architectures for both encoder and decoder, varying $\lambda_d = \lambda_{od} \in \{0, 10, 50\}$. We evaluate how well the learned latent dimensions recover the ground-truth factors. Treating KAN as the state-of-the-art architecture for functionally interpretable machine learning, we demonstrate that MKANs add \emph{monotone latent alignment} on top of KAN's per-edge functional transparency in the regime where ground-truth factors act monotonically on observations.

We adopt Spearman's rank correlation coefficient $r$ between per-sample sums of learned and ground-truth latents as our primary metric, which directly quantifies the strength of monotonic correspondence and is a natural fit for this dataset's monotone-generative structure. As secondary diagnostics, we report the DCI score~\cite{eastwood2018framework}, measuring the degree to which each latent dimension captures a single generative factor, and reconstruction error computed as the $\ell_2$ norm between reconstructed and original images.

\begin{figure}[!hh]
\centering
\includegraphics[width=0.7\columnwidth]{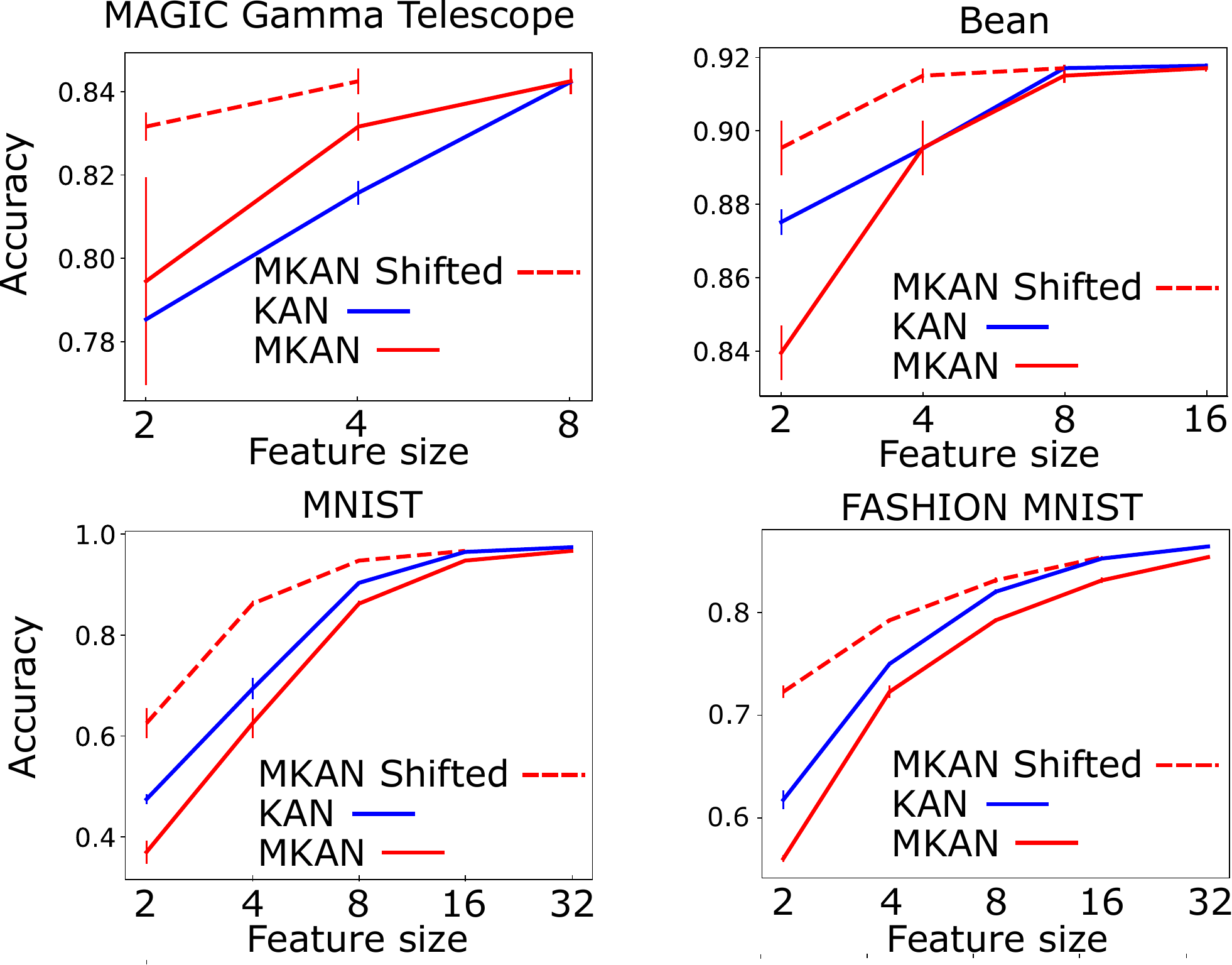}
    \caption{Figure represents classification accuracy of MKANs versus KANs within VAE framework across multiple feature sizes, four datasets and two classification metrics. Blue lines represent KANs, while red lines represent MKANs, with a dashed line standing for MKAN with a twice-increased feature space size. This comparison directly validates Theorem \ref{thm:preservation}.}
    \label{fig:kan_mkan_acc}
\end{figure}

\begin{table*}[t]
\caption{Best values of Spearman's $r$, Reconstruction MSE and DCI scores for KAN, MKAN, SMM, MLP, and Linear approaches under DIP-VAE, across $\{0,10,50\}$ $\lambda_d = \lambda_{od}$ values respectively.}
\label{tab:synthetic}
\centering
\begin{tabular}{lccc}
\toprule
    &   \multicolumn{3}{c}{\textbf{DIP-VAE}} \\
\cmidrule(lr){2-4}
Method & \makecell{Reconstruction \\ $l_2$ norm $\downarrow\downarrow$} & \makecell{Spearman's $r$  \\ $\uparrow\uparrow$}  & \makecell{DCI \\ $\uparrow\uparrow$}\\
\midrule
MLP   &  $\mathbf{3.7\pm 0.1}$          & $40.1 \pm 24.1$ & $27.6 \pm 11.8$         \\
Linear &  $74.9 \pm 0.4$          & $28.0 \pm 7.4$  & $19.1 \pm 3.6$         \\
KAN    &   $\mathit{5.8 \pm 4.1}$          & $49.4 \pm 9.7$  & $10.1 \pm9.5$   \\
\midrule
SMM   &  $37.0 \pm 7.7$  & $\mathit{80.1 \pm 1.6}$ & $\mathit{68.0 \pm 17.9}$\\
\textbf{MKAN} &  $12.1 \pm 0.8$  & $\mathbf{83.0 \pm 0.4}$ & $\mathbf{72.6 \pm 0.1}$\\
\bottomrule
\end{tabular}
\end{table*}

Table~\ref{tab:synthetic} shows that monotonic architectures (MKAN and SMM) achieve better DCI and Spearman's $r$ scores than unconstrained ones (KAN, Linear, and MLP), indicating that monotonic constraints serve as a powerful inductive bias for reflecting underlying semantics in the learned feature space. Unconstrained architectures may yield better reconstruction scores due to their wider parameter search space, even if the learned representations do not faithfully capture the data's generative structure. Notably, MKAN achieves Spearman's $r$ and DCI scores of $83$ and $72.6$, significantly outperforming KAN at $49.4$ and $10$, while KAN obtains a better reconstruction score of $5.8$ versus $12.1$. 

\subsubsection{Representation cost of monotonicity (Validating Theorem~\ref{thm:preservation})}
\label{sec:exp_ssl_sweep}

As shown in Section~\ref{sec:mkan_arch}, MKANs guarantee hard monotonicity while inheriting KAN's per-edge functional transparency. Here, we empirically evaluate the performance difference between MKANs and KANs in a self-supervised classification setting on general data.

Theorem~\ref{thm:preservation} predicts that an MKAN with $2N^*$ features should achieve classification performance comparable to or better than a KAN with $N^*$ features. We evaluate both architectures within a VAE framework across four datasets and feature sizes $\{2, 4, 8, 16, 32\}$, using MNIST, Fashion MNIST, Dry Bean~\cite{dry_bean_602}, and MAGIC Gamma Telescope~\cite{magic_gamma_telescope_159}. These datasets were selected to ensure monotonic feature transformations carry genuine semantic meaning across a range of scales and class cardinalities. Classification is performed via KNN ($K=10$) on learned features, with VAEs trained for 200 epochs at lr $0.001$.

Figure~\ref{fig:kan_mkan_acc} reports Accuracy score across feature sizes, with blue and red lines representing KAN and MKAN respectively; dashed lines show the MKAN curve shifted one tick left to align $2N$ MKAN features with $N$ KAN features. For Dry Bean and MNIST, KAN performance at a given feature size falls between MKAN at the same and twice the size, directly validating Theorem~\ref{thm:preservation}. Moreover, since MKAN with $2N$ features significantly outperforms KAN with $N$ features, MKAN can match KAN's performance with fewer than $2N$ features, justifying the benefit of preserving monotonicity without fully decomposing components as in Theorem~\ref{thm:preservation}. For MAGIC Gamma Telescope, MKAN surpasses KAN even at equal feature sizes, which we attribute to improved extrapolation properties of monotonic representations.

\section{Discussion and Limitations}
\label{sec:discussion}
\textbf{Zero violation rate has a measurable downstream effect.} By-construction monotonicity has a measurable consequence MonoKAN cannot match: zero violation rate throughout training versus constrained-subset enforcement maintained by projection. On the synthetic ellipses dataset, where the ground truth is strictly monotone, this translates into a Spearman alignment $83.0$ vs.\ $49.4$ under DIP-VAE, with DIP-VAE DCI $72.6$ vs.\ $10.1$ (Table~\ref{tab:synthetic}). Unconstrained baselines fit reconstruction better but do not recover monotone structure which represents a separation only available to architectures that encode the prior in their parameterization.

\textbf{Monotonicity is a quantifiable trade-off, not a categorical loss.} Theorem~\ref{thm:preservation} reframes the cost of monotonicity as an embedding-dimension trade-off. The Dry Bean, MAGIC Gamma Telescope, MNIST, and Fashion MNIST curves in Figure~\ref{fig:kan_mkan_acc} sit between MKAN at $N^*$ and MKAN at $2N^*$, validating the prediction; on MAGIC Gamma Telescope MKAN matches KAN at equal feature size, suggesting the constraint acts as a beneficial bias rather than an overhead in monotone-structured data.

\paragraph{Limitations.} (i) MKAN lacks an adaptive grid range. (ii) Theorem~\ref{thm:preservation} is an existence result, not a learning guarantee. (iii) The analysis assumes ball-shaped neighborhoods; ellipsoidal neighborhoods (Definition~\ref{def:neighborhood}) are a natural extension. (iv) Self-supervised experiments use $\beta$-VAE / DIP-VAE; FactorVAE-style total-correlation penalties~\cite{duan2022factorvae} are unexplored. (v) The exponential reparameterization Eq.~\eqref{eq:mkan_layer} can be ill-conditioned for very large $\omega$; weight decay sufficed here, but principled treatment may be warranted for deeper stacks.

\section{Conclusion}
\label{sec:conclusion}
The proposed MKAN establishes architectural monotonicity in Kolmogorov--Arnold Networks as a hard, projection-free, gradient-friendly inductive bias representing the first KAN variant whose monotonicity holds for every parameter value. Theorem~\ref{thm:preservation} characterizes the embedding-dimension cost of this constraint precisely: at most $2\times$, validated empirically on four real classification datasets and exact (rather than worst-case) on data with strong monotonic structure. Together, these results recast monotonicity as a quantifiable representation-cost trade-off rather than a categorical restriction, and identify MKAN as the natural KAN-based realization that combines hard monotonicity with per-edge transparency. Such a combination is unavailable to MLP- or flow-based monotone networks. Natural next steps are partial monotonicity via per-edge masks, ellipsoidal-neighborhood theorems, and identifiability under monotone-generative data.

\bibliographystyle{plainnat}
\bibliography{references}

\newpage
\appendix
\setcounter{definition}{0}
\setcounter{theorem}{0}

\section{MKAN architecture}

\subsection{Enforcing Monotonicity in KANs (MKAN)}
\label{app:mkan_arch}
In this subsection, we describe how to enforce monotonicity in KANs to enhance interpretability. 
To enforce monotonicity, we require: (1) each spline $\phi'_{ij}$ to be monotonically increasing, (2) all scaling weights to be positive, and (3) the base activation to be monotonically increasing. This structure enhances functional interpretability: each node in the network is affected by other nodes in the same monotonic way, e.g., consistently increasing or decreasing, making signal propagation directly traceable. Moreover, when the data admit monotone semantic transformations (e.g., adding active pixels to a binary image moves it toward another class), the overall monotone input--output map of the network induces \emph{monotone latent alignment}: such transformations correspond to predictable, monotone directions in feature space.

\paragraph{Monotonic splines via coefficient ordering.}
A cubic B-spline $\phi'(x) = \sum_{k=1}^{K+p} \gamma'^{(k)} h_k(x)$ is guaranteed to be monotonically increasing if its coefficients are strictly ordered~\citep{deboor1978practical, wang2025monotone}:
\begin{equation}
\label{eq:spline_cond}
\gamma'^{(1)} < \gamma'^{(2)} < \cdots < \gamma'^{(K+p)}.
\end{equation}
We enforce this condition through an \emph{exponential reparameterization}. Instead of directly learning the constrained $\gamma'$ coefficients, we introduce unconstrained parameters $\omega_{ij}^{(k)} \in \mathbb{R}$ and define:
\begin{equation}
\label{eq:reparam}
\gamma_{ij}'^{(1)} = \omega_{ij}^{(1)}, \qquad
\gamma_{ij}'^{(k)} = \omega_{ij}^{(1)} + \sum_{t=2}^{k} \exp\!\big(\omega_{ij}^{(t)}\big), \quad k \in \{2, \ldots, K+p\}.
\end{equation}
Since $\exp(\omega) > 0$ for all $\omega \in \mathbb{R}$, the condition in Eq.~\eqref{eq:spline_cond} is automatically satisfied. The inverse mapping from $\gamma'$ to $\omega$ is:
\begin{equation}
\label{eq:inverse}
\omega_{ij}^{(1)} = \gamma_{ij}'^{(1)}, \qquad
\omega_{ij}^{(k)} = \log\!\big(\gamma_{ij}'^{(k)} - \gamma_{ij}'^{(k-1)}\big), \quad k \in \{2, \ldots, K+p\}.
\end{equation}

The weights $\omega_{ij}^{(k)} $ are initialized to make corresponding coefficients  $\gamma_{ij}^{'(k)}$ to increase and have a scale of $\sigma$  typically of order $0.05$ as in the original KAN paper \cite{liu2025kan}. We achieve this by firstly generating a normal noise $n_{ij}^{(k)}$ and making it monotonic across grid dimension:
\begin{align}
&\gamma_{ij}^{'(1)} = n_{ij}^{(1)} \\
&\gamma_{ij}^{'(k)} = n_{ij}^{(1)}+ \sum_{t=2}^{k} \exp\!\big(n_{ij}^{(t)}\big)
\end{align}

The final step is to normalize these coefficients: \begin{equation}
\gamma_{ij}^{'(k)} = (\gamma_{ij}^{'(k)}-\frac{(K+p-1)\sqrt{e}}{2})\frac{2\sigma}{(K+p-1)\sqrt{e}}
\end{equation}
and derive weights $\omega_{ij}^{(k)} $ using Eq. \ref{eq:inverse}. This normalization insures that \( \mathbb{E}\gamma_{ij}^{'(K+p)} = -\mathbb{E}\gamma_{ij}^{'(1)} = \sigma
\)  because $\mathbb{E}e^{n_{i,j}^{(k)}} = \sqrt{e}$. 

\paragraph{MKAN layer.}
Combining monotonic splines with positive weight constraints and a monotonic activation, the MKAN layer is defined as:
\begin{equation}
\label{eq:mkan_layer}
F'^{(j)}(x_1, \ldots, x_{N_\text{in}}) = \sum_{i=1}^{N_\text{in}} \Big( \exp(w'^{(s)}_{ij}) \, \phi'_{ij}(x_i) + \exp(w'^{(b)}_{ij}) \, \mathrm{ReLU}(x_i) \Big) + b_j,
\end{equation}
where the exponential ensures positivity of the effective weights. The bias term $b_j$ is necessary because $\exp(w) \cdot \mathrm{ReLU}(x) \geq 0$ for all $x$, so without bias the layer output would be bounded below. We initialize these base output coefficients according to \cite{hoedt2023principled}. Stacking MKAN layers yields a monotonically increasing feature extractor, since the composition of monotonically increasing functions is monotonically increasing.

\paragraph{Parameter count.} An MKAN layer has $(K+p) \cdot N_\text{in} \cdot N_\text{out}$ spline parameters (same as KAN) plus $N_\text{in} \cdot N_\text{out}$ base weight parameters and $N_\text{out}$ bias parameters. The overhead compared to KAN is exactly $N_\text{out}$ parameters for the bias.

\section{Full Proofs}
\label{app:proofs}
In this section, we address a fundamental question: how do imposed monotonicity constraints affect the ability of FE to measure semantic similarity in the feature space using a distance metric? Semantic structures in feature space are referred to as semantic neighborhoods, as they capture similarity between samples according to a chosen distance metric \cite{fu2020semantic}. For theoretical tractability, we assume existence of arbitrary FE, which feature space can be partitioned 
with ball-shaped semantic neighborhoods. We then prove that exists a monotonic FE whose feature space can be equivalently partitioned into semantic neighborhoods. We apply this knowledge to compare the performance of KAN, used as FE, and MKAN utilized as monotonic FE. Figures~\ref{fig:arb_example} and~\ref{fig:mono_example} illustrate an example in which semantic neighborhoods induced by a FE can be reproduced by a monotonic FE, potentially at the cost of a higher embedding dimension. We start theoretical analysis with formal definitions of feature extractor, dataset and semantic neighborhood partition.

\begin{figure*}[!ht]
\hspace*{1cm}
\centering
\label{fig:mono_arb_example}
\begin{subfigure}[b]{0.35\textwidth}
    \includegraphics[width=\textwidth]{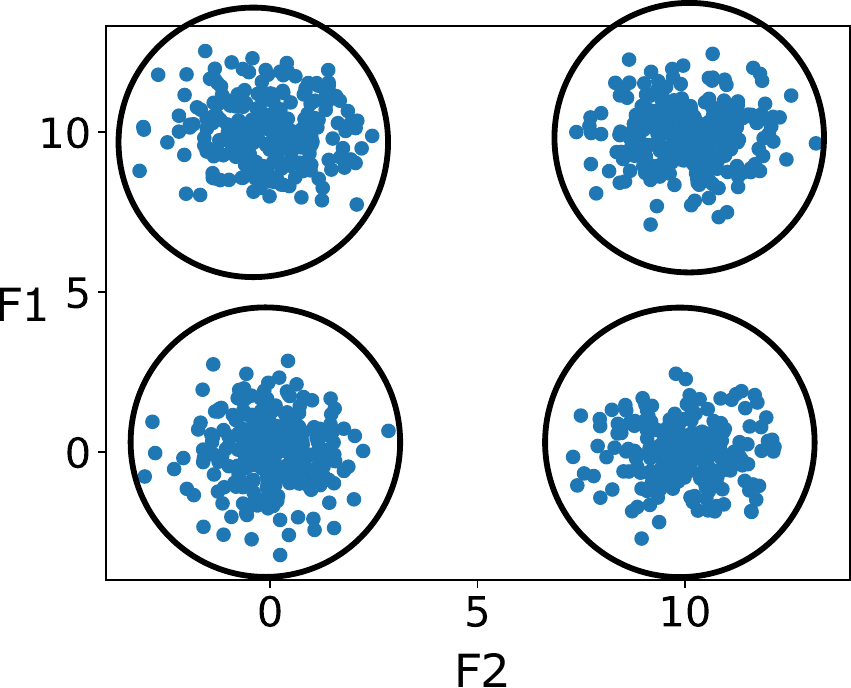}
    \caption{Samples in a 2D features space of arbitrary FE: $F_1 = x_1, F_2=x_2-x_3$. }
    \label{fig:arb_example}
\end{subfigure}
\hfill
\centering
\begin{subfigure}[b]{0.4\textwidth}
    \includegraphics[width=\textwidth]{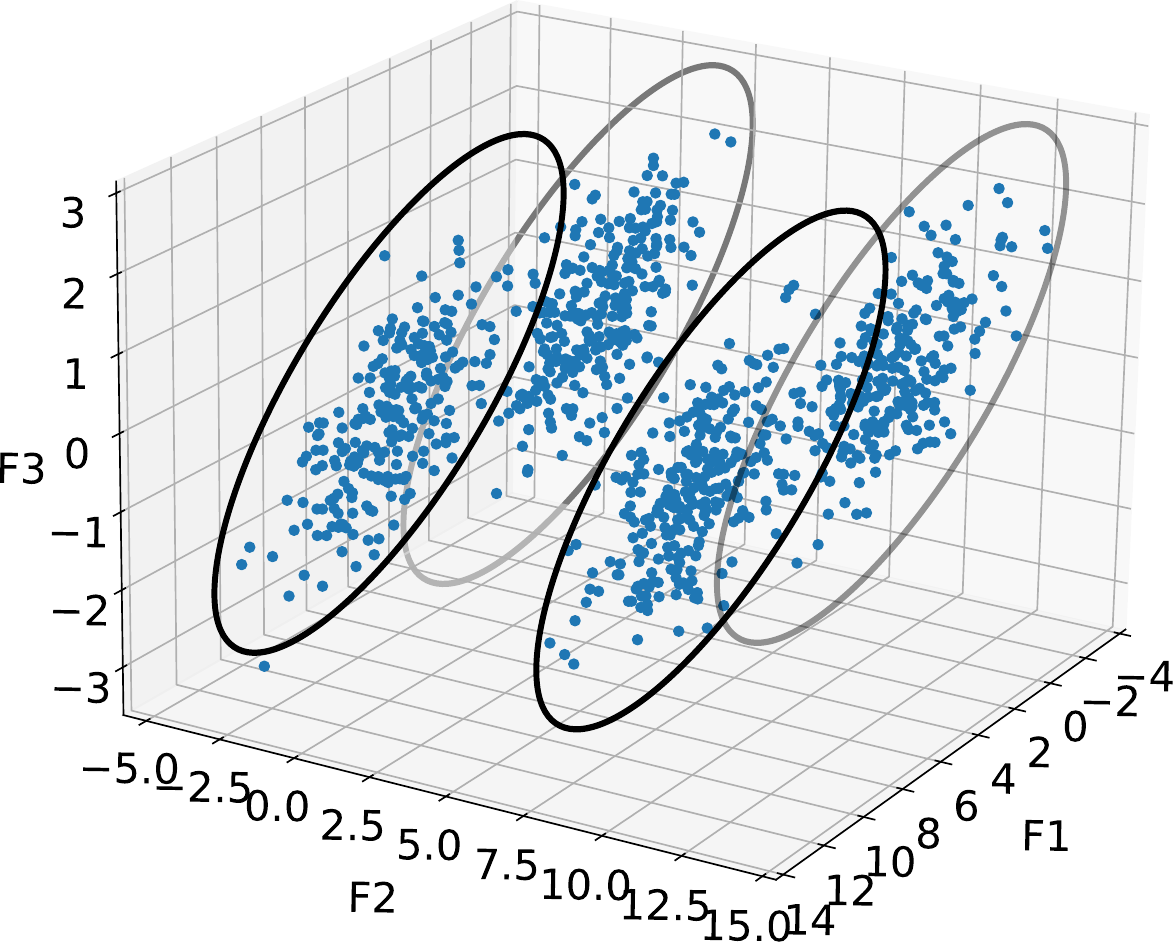}
    \caption{Samples in a 3D features space of monotonic FE: $F_1 = x_1, F_2=x_2, F_3=x_3$. }
    \label{fig:mono_example}
\end{subfigure}
\hspace*{1cm}
\caption{An example of how semantic neighborhoods in arbitrary features space can be reproduced in a monotonic one. The input data consist of three-dimensional vectors $(x_1, x_2, x_3)$, 
where $x_1 \sim 0.5\,\mathcal{N}(0,1) + 0.5\,\mathcal{N}(10,1)$, 
$x_3 \sim \mathcal{N}(0,1)$, and 
$x_2 = r + x_3 \:, r \sim 0.5\,\mathcal{N}(0,1) + 0.5\,\mathcal{N}(10,1)$.}

\label{fig:theor_dem}
\end{figure*}

\subsection{Definitions}

\begin{definition}[Feature extractor]
\label{def:fe}
A \emph{feature extractor} is a $C^K, K > 0 $ mapping $F: [-1,1]^N \to \mathbb{R}^{N^*}$, $F = \{F^j\}_{j=1}^{N^*}$.
\end{definition}

One can observe that this definition excludes most common neural networks, as they are not differentiable at specific points due to activation functions such as ReLU. Indeed, the ReLU function is differentiable everywhere except at zero, where its derivative is zero by convention to enable backpropagation. However, in general theoretical studies such as optimization of neural networks \cite{chatterjee2022convergence, du2019gradient, arora2018convergence} the neural network is usually considered to be smooth as non-differentiable activation functions like ReLU can be approximated by smooth maps. 

\begin{definition}[Monotonic feature extractor]
\label{def:mono_fe}
A feature extractor $F$ is \emph{monotonically increasing} if for all $i \in \{1,\ldots,N\}$ and $j \in \{1,\ldots,N^*\}$, the partial derivative $\frac{\partial F^j}{\partial x_i}(x) \geq 0$ for all $x \in [-1,1]^N$.
\end{definition}

\begin{definition}[Semantic neighborhood]
\label{def:neighborhood}
A \emph{semantic neighborhood} $C_R(x_c)$ with radius $R$ and center $x_c$ in $\mathbb{R}^{N^*}$ is the open set $C_R(x_c) = \{x \in \mathbb{R}^{N^*} : (x - x_c)^\top \mathbf{A} (x - x_c) < R^2\}$, where $\mathbf{A} \in \mathbb{R}^{N^* \times N^*}$ is positive definite. In the special case $\mathbf{A} = \mathbf{I}$, the neighborhood reduces to an 
open Euclidean ball is called ball-shaped semantic neighborhood. 
\end{definition}

\begin{definition}[Dataset]
\label{def:dataset}
\emph{Dataset} $D$ is a finite set of pairs \[ \{(x_i, l_i)\}, \quad x\in \mathbb{R}^N, l\in \{ 1, \dots,T \}\] where $x$ is one training instance (i.e. input vector) and $l$ is a label and $T$ is a number of classes.
\end{definition}

\begin{definition}[Semantic neighborhood partition]
\label{def:partition}
A feature extractor $F$ \emph{induces a semantic neighborhood partition} of a dataset $D$ if there exist disjoint semantic neighborhoods $\{C_i\}_{i=1}^{C}$ in $\mathbb{R}^{N^*}$ such that for every $x \in D$, there exists a unique $i$ with $F(x) \in C_i$.
\end{definition}

\subsection{Proof of Theorem~\ref{thm:decomp} (Monotonic Decomposition)}
\begin{theorem}[Monotonic decomposition]
\label{thm:decomp}
Let $F:[-1,1]^N \to \mathbb{R}^{N^*}$ be a $C^K$ feature extractor with
$K > o$. Then $F$ admits a decomposition
\[
    F(x) \;=\; F_{\mathrm{inc}}(x) \;+\; F_{\mathrm{dec}}(x),
\]
where $F_{\mathrm{inc}}$ is coordinate-wise increasing and
$F_{\mathrm{dec}}$ is coordinate-wise decreasing.
\end{theorem}

\begin{proof}
The decomposition is constructed componentwise, so it suffices to treat
the scalar case; fix $j \in \{1,\dots,N^*\}$ and write $f := F^j$.

\medskip
\emph{Step 1: Decomposition on $[0,1]^N$.}
We first prove the result on the unit cube. Let
$\tilde f:[0,1]^N \to \mathbb{R}$ be $C^K$. By Taylor's theorem with
remainder, expanded around the origin,
\begin{equation}
    \tilde f(y)
    \;=\;
    \underbrace{\sum_{|\alpha|<K} c_{\alpha}\, y^{\alpha}}_{=:\,P(y)}
    \;+\;
    \underbrace{\sum_{|\alpha|=K} h_{\alpha}(y)\, y^{\alpha}}_{=:\,R(y)},
\end{equation}
where $c_{\alpha} \in \mathbb{R}$ and the coefficients $h_{\alpha}$ are
continuous on $[0,1]^N$. We decompose $P$ and $R$ separately.

\smallskip
\emph{Polynomial part.}
On $[0,1]^N$ every monomial $y^{\alpha}$ is coordinate-wise increasing.
Splitting $P$ according to the sign of its coefficients,
\[
    P_{\mathrm{inc}}(y) := \sum_{c_{\alpha}>0} c_{\alpha}\,y^{\alpha},
    \qquad
    P_{\mathrm{dec}}(y) := \sum_{c_{\alpha}<0} c_{\alpha}\,y^{\alpha},
\]
yields $P = P_{\mathrm{inc}} + P_{\mathrm{dec}}$ with $P_{\mathrm{inc}}$
increasing and $P_{\mathrm{dec}}$ decreasing.

\smallskip
\emph{Remainder part.}
Since $R$ is $C^1$ on the compact cube $[0,1]^N$, the constant
\[
    L \;:=\; \max_{1\leq i \leq N}\;\sup_{y\in[0,1]^N}
            \left|\frac{\partial R}{\partial y_i}(y)\right|
\]
is finite. Setting $S(y) := L\,(y_1 + \cdots + y_N)$, we write
\[
    R(y) \;=\; \underbrace{\bigl(R(y)-S(y)\bigr)}_{R_{\mathrm{dec}}(y)}
            \;+\; \underbrace{S(y)}_{R_{\mathrm{inc}}(y)}.
\]
Here $R_{\mathrm{inc}}$ is increasing because
$\partial S/\partial y_i = L \geq 0$, and $R_{\mathrm{dec}}$ is decreasing
because, by the choice of $L$,
$\partial(R-S)/\partial y_i = \partial R/\partial y_i - L \leq 0$.

Defining
$\tilde f_{\mathrm{inc}} := P_{\mathrm{inc}} + R_{\mathrm{inc}}$ and
$\tilde f_{\mathrm{dec}} := P_{\mathrm{dec}} + R_{\mathrm{dec}}$
gives the desired decomposition on $[0,1]^N$.

\medskip
\emph{Step 2: Transfer to $[-1,1]^N$.}
Let $\varphi:[-1,1]^N \to [0,1]^N$ be the affine bijection
$\varphi(x) := \tfrac{x+1}{2}$, and define
$\tilde F(y) := F(2y-1)$ on $[0,1]^N$. Applying Step~1 componentwise
to $\tilde F$ produces
$\tilde F = \tilde F_{\mathrm{inc}} + \tilde F_{\mathrm{dec}}$. Set
\[
    F_{\mathrm{inc}}(x) := \tilde F_{\mathrm{inc}}\bigl(\varphi(x)\bigr),
    \qquad
    F_{\mathrm{dec}}(x) := \tilde F_{\mathrm{dec}}\bigl(\varphi(x)\bigr).
\]
Then $F = F_{\mathrm{inc}} + F_{\mathrm{dec}}$, and monotonicity is
preserved under this change of coordinates: by the chain rule, since
$\partial \varphi_i/\partial x_i = 1/2 > 0$,
\[
    \frac{\partial F_{\mathrm{inc}}}{\partial x_i}(x)
    \;=\;
    \tfrac{1}{2}\,
    \frac{\partial \tilde F_{\mathrm{inc}}}{\partial y_i}\!\bigl(\varphi(x)\bigr)
    \;\geq\; 0,
\]
and analogously $\partial F_{\mathrm{dec}}/\partial x_i \leq 0$.
\end{proof}

\subsection{Proof of Lemma~\ref{lem:ellipsoid}}

\begin{lemma}
\label{lemma:elips}
\(\forall  R>0,N^*>0, 0 \le k \le N^*, x_c\in\mathbb{R}^{N^*+k}\) if a finite set of points $P$ belong to the set $C^{\rho_s}_R(x_c)=\{x \in \mathbb{R}^{N^*+k}: \rho_s(x, x_c) < R\}$, where $\rho_s$ is defined in Eq. \ref{eq:semi_metric}, then exists a semantic neighborhood $C_R(x_c)$ from Definition \ref{def:neighborhood}, which belongs to $C^{\rho_s}_R(x_c)$ and contains all points from the finite set $P$. 
\begin{equation}
\label{eq:semi_metric}
\rho_s(x,x')^2 = 
\sum_{j=1}^{k} \left( (x_{2j} - x'_{2j}) - (x_{2j-1} - x'_{2j-1}) \right)^2 
+ \sum_{j=k+1}^{N^*} (x_{j+k} - x'_{j+k})^2 
\end{equation}
\end{lemma}

\begin{proof}
    We construct the semantic neighborhood $C_R(x_c)$ as 
    \[
    x: (x - x_c)^T\mathbf{A}(x-x_c) < R^2,
    \] where \( \mathbf{A} = \mathbf{R^T DR}\), 
    with $\mathbf{D}$ being a positive diagonal matrix and $\mathbf{R}$ being a rotation matrix, which rotates
    planes $(x_i, x_{i-1})$ \(i \in \{1, \dots,k\}\) on 45 degrees counterclockwise: \( ( \frac{x_i-x_{i-1}}{\sqrt{2}}, \frac{x_i+x_{i-1}}{\sqrt{2}})\) and does not change other coordinates. This rotation is needed to make the neighborhood definition related to the semi-metric $\rho_s(x,x_c)$. Indeed, taking into account that \[((x - x_c)^T\mathbf{R^T DR}(x-x_c)= (\mathbf{R}(x-x_c))^T\mathbf{D}\mathbf{R}(x-x_c) \]
    The neighborhood definition can be now rewritten as:
    \begin{align*}
    &D_1\sum_{i=k+1}^{N^*}(x_{i+k}-x_{(i+k)c})^2  
    +\frac{D_2}{2}\sum_{i=1}^{k}
    (x_{2i}-x_{2ic} - (x_{2i-1}-x_{(2i-1)c}))^2 +\\
    &+\frac{D_3}{2}\sum_{i=1}^{k}
    (x_{2i}-x_{2ic} + (x_{2i-1}-x_{(2i-1)c}))^2 < R^2,
   \end{align*} where $D_1, D_2, D_3$ are coefficients of the diagonal matrix $\mathbf{D}$: 
   \[ D_1= \mathbf{D_{i+k,i+k}}, i\in\{k+1,N^*\}\]
   \[
   D_2 = \mathbf{D_{2i,2i}}, D_3 = \mathbf{D_{2i-1,2i-1}},  i\in\{1,k\} \]
   We set $D_1=1, D_2=2$ and $D_3=2\beta$, where $\beta>0$, to express this neighborhood inequality using $\rho_s$. The  inequality is then represented as:
   \begin{equation}
   \label{eq:lemma}
f(x,x_c,\beta)=\rho^2_s(x,x_c)+\beta\sum_{i=1}^{k}
    (x_{2i}-x_{2ic} + (x_{2i-1}-x_{(2i-1)c}))^2 < R^2
\end{equation}

   Let $B$ be the maximal value of the $\rho_s(x, x_c)^2$ on the finite set $P$. All points from $P$ belong to the $C_R^{\rho_s}(x_c)$, thus $B < R^2$. The second part of the equation:
   \[
   \beta\sum_{i=1}^{k}
    (x_{2i}-x_{2ic} + (x_{2i-1}-x_{(2i-1)c}))^2
   \]
   is also restricted on the finite set $P$. Therefore, we can choose sufficiently small $\beta$ parameter to make inequality \ref{eq:lemma} true for all points from $P$. This means that the set $P$ belongs to the neighborhood we constructed. Taking into account that \(f(x, x_c, \beta) > \rho_s^2(x,x_c)\) and \[\forall x \:f(x,x_c,\beta) < R^2 \rightarrow \rho_s(x,x_c) < R^2\]
   our neighborhood also belongs to the set $C^{\rho_s}_R(x_c)=\{x \in \mathbb{R}^{N'}: \rho_s(x, x_c) < R\}$ from the Lemma statement.
\end{proof}

\subsection{Proof of Theorem~\ref{thm:preservation} (Representation Cost of Monotonicity)}

\begin{figure}[!th]
\centering
    \includegraphics[width=0.6\columnwidth]{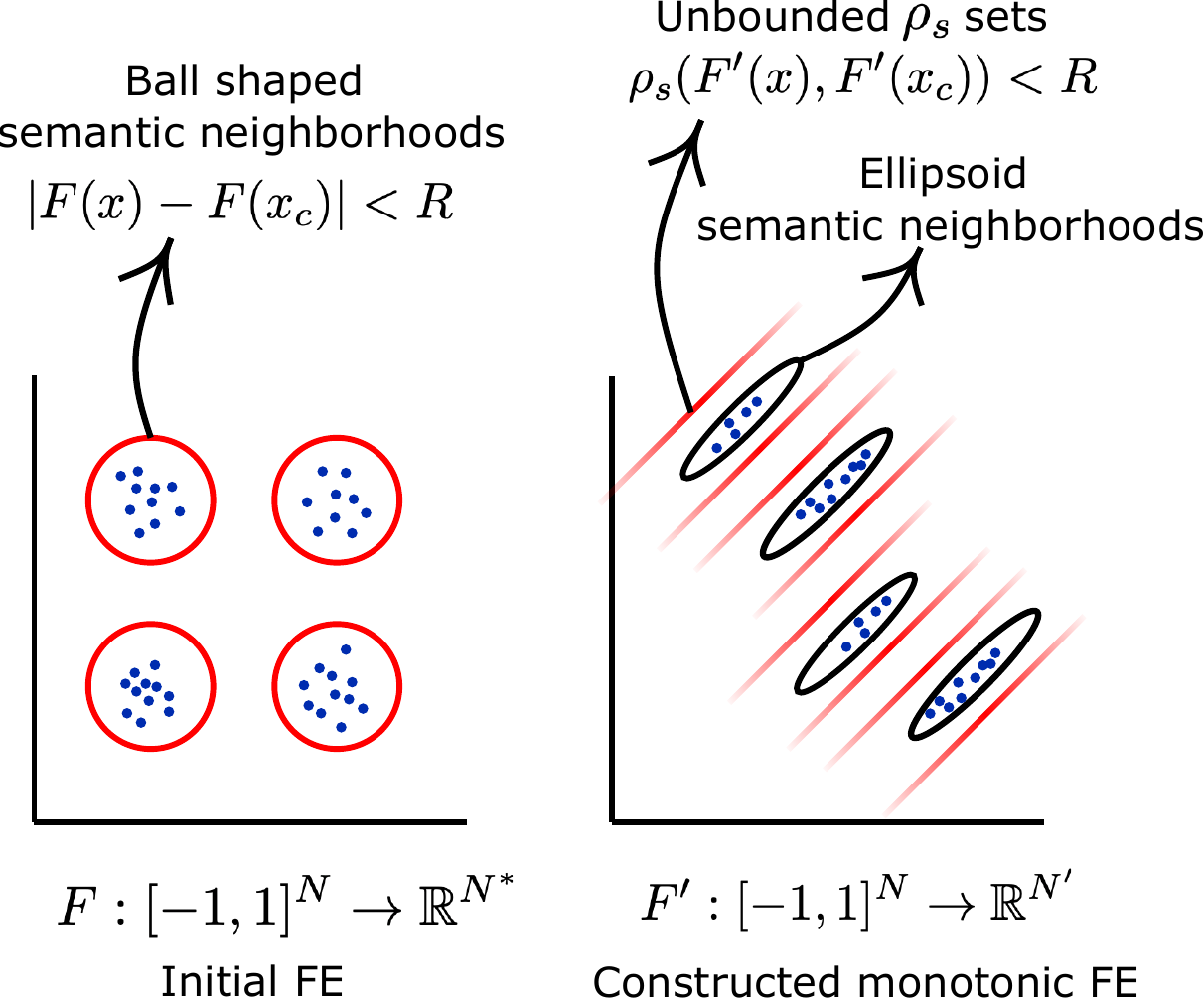}
    \caption{Preservation of the semantic neighborhood partition under the constructed monotonic FE (F'). Red circles denote the original ball-shaped semantic neighborhoods in the feature space of (F). Red contours illustrate the corresponding unbounded sets defined by the semi-metric ($\rho_s$), into which these neighborhoods are mapped in the feature space of the constructed monotonic (F'). Black ellipses represent the semantic neighborhoods within the ($\rho_s$)-induced sets that contain all associated samples.}
    \label{fig:two_clusterings}
\end{figure}

\begin{theorem}[Representation cost of monotonicity]
\label{thm:preservation}
Let $F: [-1,1]^N \to \mathbb{R}^{N^*}$ be a feature extractor that induces a semantic neighborhood partition of a dataset $D$ with ball-shaped neighborhoods $\{C_i\}_{i=1}^{C}$, and let $0 \le k \le N^*$ be the number of non-monotone coordinate functions of $F$. Then there exists an increasing feature extractor $F': [-1,1]^N \to \mathbb{R}^{N'}$ with $N' = N^* + k \leq 2N^*$ and semantic neighborhoods $\{C'_i\}_{i=1}^{C}$ with centers $c_i$ in $\mathbb{R}^{N'}$ that induce the same partition of $D$ as shown in the Figure \ref{fig:two_clusterings}.
\end{theorem}

\begin{proof}
Given a FE
\(
F : \mathbb{R}^N \rightarrow \mathbb{R}^{N^*},
\)
let
\(
\{F^{(j)}_{\mathrm{inc}},\, F^{(j)}_{\mathrm{dec}}\}_{j=1}^{N^*}
\)
denote the decomposition increasing and decreasing FEs  guaranteed by Theorem~\ref{thm:decomp}. 
Let \(0 \le k \le N^*\) be the number of non-monotonic coordinate functions 
\(F^{(j)}\); without loss of generality, assume that these correspond to indices 
\(j \in \{1,\dots,k\}\). The construction of the monotonic FE \(F'\) proceeds as follows. 
Each non-monotonic coordinate \(j \in \{1,\dots,k\}\) is decomposed into the pair
\(
\{F^{(j)}_{\mathrm{inc}},\, -F^{(j)}_{\mathrm{dec}}\},
\)
yielding two  increasing coordinates. For the remaining 
\(N^* - k\) monotonic coordinates, increasing coordinates are left unchanged, 
while decreasing coordinates are multiplied by \(-1\) to make them increasing. As a result, the constructed increasing FE consists of
\[
N' = 2k + (N^* - k) = N^* + k \le 2N^*
\]
coordinates. The resulting increasing FE \(F'\) is formally 
defined as follows:

\begin{equation}
\label{eq:const}
\begin{aligned}
&F' :\; \mathbb{R}^N \rightarrow \mathbb{R}^{N^* - k + 2k} = \mathbb{R}^{N'}
\\
&F'^{(2j-1)} = -F^{(j)}_{\mathrm{dec}}, 
F'^{(2j)} = F^{(j)}_{\mathrm{inc}},
j \in \{1,\dots,k\}
\\
&F'^{(k+j)} =
\begin{cases}
F^{(j)},  \text{if } F^{(j)} \text{ is increasing}, \\[0.2em]
-\,F^{(j)},  \text{if } F^{(j)} \text{ is decreasing},
\end{cases}
 j \in \{k+1,\dots,N^*\}.
\end{aligned}
\end{equation}
By construction, for all \(j \in \{1,\dots,k\}\), the odds indexed components are $F'^{(2j-1)} = -F^{(j)}_{\mathrm{dec}}$ and the even indexed components are $F'^{(2j)} = F^{(j)}_{\mathrm{inc}}$. Given that $F_{inc}^{(j)} + F_{dec}^{(j)} = F^{(j)}$, the following expression is true:
\begin{equation}
\label{eq:diff}
F'^{(2j)} - F'^{(2j-1)} = F^{(j)}
\end{equation}
Assume that there is a ball-shaped semantic neighborhood from Definition \ref{def:neighborhood} in the $F$ feature space from. The distance between representations of some data point $x \in R^N$ and neighborhood center $c \in R^N$ is expressed as: 
\begin{equation}
    ||F(x) - F(c)||^2 = \sum_{j=1}^{N^*} (F^{(j)}(x) - F^{(j)}(c))^2
\end{equation} 
This implies the Euclidean distance in $F$-space equals the semi-metric $\rho_s$, Eq. \ref{eq:semi_metric} in $F'$-space:
\[
\|F(x) - F(c)\|^2 = \rho_s(F'(x), F'(c))^2.
\]

We define prototype neighborhoods $C'_{is} = \{x \in \mathbb{R}^{N'} : \rho_s(x, F'(c_i)) < R\}$ in $F'$-space. By the distance equality, these sets have the same sample assignments as the original neighborhoods:
\[
\forall x \in D,\ i \in \{1,\ldots,C\}: \quad F(x) \in C_i \iff F'(x) \in C'_{is}.
\]

By Lemma~\ref{lem:ellipsoid}, for each $C'_{is}$ there exists a proper semantic neighborhood $C'_i$ (Definition~\ref{def:neighborhood}) that contains all data points from $C_i$ and is contained within $C'_{is}$. Therefore, the neighborhoods $\{C'_i\}_{i=1}^C$ in $F'$-space induce the same partition as $\{C_i\}_{i=1}^C$ in $F$-space.
\end{proof}

\section{Additional Experimental Details}
\label{app:exp_details}
\subsection{Supervised Experiments}
In this section, we describe the MKAN configurations used in the supervised setting. An MKAN layer is specified in the format [input size, output size, grid range, grid size], and a multilayer MKAN is specified as a sequence of such layers. For fully monotonic datasets such as Synthetic, we constrain the entire MKAN layer to be monotonic, whereas for partially monotonic datasets only the edges of the first layer corresponding to monotonic inputs are monotonic, and the second layer (if present) is fully monotonic. All experiments in this section are run on AMD EPYC 75F3 CPU, except those on the ChestXRay dataset, which are run on an A100 GPU. The precise training parameters like $lr$, batch size etc. are left untouched with respect to \cite{igel2024smooth}.  Table~\ref{tab:mkans_supervised_config} reports the MKAN configurations and parameter counts for each dataset; we vary only the number of layers and the hidden dimension size. A single run took at most 10 minutes. 

\begin{table}[h]
    \centering
    \begin{tabular}{|c|c|c|}
         Dataset & MKAN Configuration  & Parameters \\ \hline
         Synthetic & [$N_{in}$, $N_{in}$, (0,1), 5], [$N_{in}$, 1, (-3,3), 5] & 45, 145, 301 \\ \hline
         COMPAS & [13, 3, (-2,2), 5], [3, 1, (-3,3), 5] & 92\\ \hline
         LoanDefaulter & [28, 10, (-2,2), 5], [10, 1, [
         (-3, 3), 5] & 2041\\ \hline
         ChestXRay & [130, 10, (-2,2), 5], [10, 1, [
         (-3, 3), 5] & 9181 \\ \hline
         Heart Disease & [13, 1, (-2,2), 5] & 92 \\ \hline
         BlogFeedback & [276, 2, (-2,2), 5], [2, 1, (-3,3), 5] & 3881 \\  \hline
         AutoMPG & [8, 3, (-2,2), 5], [3, 1, (-3,3), 5] & 193 \\ \hline
    \end{tabular}
    \caption{The table represents MKANs configurations and parameters numbers for dataset used in supervised evaluations.}
    \label{tab:mkans_supervised_config}
\end{table}

Here we also provide a short description of the datasets: COMPAS \cite{angwin2022machine} (6172 samples, 13 inputs with 4 monotonic), LoandDefaulter (0.5M samples, 28 inputs with 5 monotonic), ChestXRay (5606 samples, 130 inputs with 2 monotonic), Heart Disease (303 samples, 13 inputs with 2 monotonic) for classification and BlogFeedback \cite{buza2013feedback} (54000 samples, 276 inputs with 8 monotonic), Auto MPG (398 samples, 8 inpits with 3 monotonic) for regression. Inputs for the ChestXRay dataset consist of given features and of features of actual images produced by ResNet18 mode and evaluations are made with fixed and trainable ResNet18.

\subsection{Self-supervised experiments}
\subsubsection{Ellipses dataset}
As a proof of concept, we construct a synthetic dataset in which, unlike those used in related work, all observed pixels depend monotonically on three ground-truth latent factors $z = (z_1, z_2, z_3)$. Each sample $x$ is a $20 \times 20$ grayscale image containing an exponentially decaying elliptical spot:
\begin{equation}
    x_{i,j} = \frac{z_3}{100} \exp\!\left(-\frac{(i-10)^2}{z_1^2} - \frac{(j-10)^2}{z_2^2}\right), \quad i,j \in \{1, \ldots, 20\},
\end{equation}
centered at $(10, 10)$ with semi-axes $r_1 = z_1$ and $r_2 = z_2$, where $z_1, z_2 \sim \mathcal{U}\{1, \ldots, 10\}$ and $z_3 \sim \mathcal{U}\{1, \ldots, 100\}$. By construction, the mapping from latent factors to observations is strictly monotone: increasing $z_1$ or $z_2$ expands the set of non-zero pixels, while increasing $z_3$ raises their intensity uniformly. The dataset comprises 8K training and 2K test samples.

We train fixed-variance ($\sigma = 0.05$) DIP-VAE framework across $\lambda_b = \lambda_{ob} \in \{0, 10, 50\}$, with learning rate $0.001$ (Adam optimizer) and batch size $256$, performing three runs of 2K epochs each. DCI, Spearman's $r$, and reconstruction $l_2$ error scores are recorded every 10 epochs, and the final values are selected for reporting with 2 $\sigma$ error bars (randomness comes from weights initialization). For each training configuration, we evaluate five architectures: three unconstrained (Linear, MLP, KAN) and two monotonic (SMM, MKAN) as both encoder and decoder. KAN and MKAN layers are specified as [input size, output size, grid range, grid size], and linear layers as [input size, output size]. Table~\ref{tab:mkans_selfsupervised_config_synt} reports the encoder and decoder configurations. A large grid size is used in the KAN and MKAN decoders to effectively reconstruct the original image with comparison to other architectures and to accurately capture the distribution of learned features, which may have a small scale relative to the grid range; the grid range itself could alternatively be tuned to permit a smaller grid size. The lack of an adaptive grid range is a current limitation of the MKAN architecture. All experiments here are performed on CPU, and a single run takes at most 40 minutes. 
\begin{table}[h]
    \centering
    \begin{tabular}{|c|c|c|c|}
        Architecture & Encoder config & Decoder config & Parameters \\  \hline
         Linear &  [400, 3] & [3, 400] & 1200,1200 \\ \hline
         MLP & [400, 30], ReLU, [30, 3] & [3, 30], ReLU, [30, 400] & 12090, 12090\\ \hline
         KAN & [400, 30, (0,1), 30], & [3, 30, (-3,-3), 200], &  387K, 2M\\
             & [30, 3, (-3,3), 3] & [30, 400, (-3.-3), 200] &  \\ \hline
         MKAN & [400, 30, (0,1), 30], & [3, 30, (-3,-3), 200],& 387K, 2M\\
              & [30, 3, (-3,3), 3]    & [30, 400, (-3.-3), 200] & \\ \hline
         SMM & [400, 3, 6,6] & [3, 400, 6, 6] &  43200, 43200 \\ \hline
        
    \end{tabular}
    \caption{The table represents encoder and decoder configurations and parameters numbers across different types of architectures. }
    \label{tab:mkans_selfsupervised_config_synt}
\end{table}

\subsubsection{Representation cost of monotonicity}
We evaluate KANs and MKANs across feature sizes $\{2, 4, 8, 16, 32\}$ on four datasets: MNIST (784 inputs), Fashion-MNIST (784 inputs), Dry Bean (16 inputs), and MAGIC (10 inputs). For each configuration (feature size, dataset, and architecture), we perform five runs of 200 epochs each, with learning rate $0.001$ (Adam optimizer) and batch size $256$, using the VAE framework with fixed variance $0.05$. For classification, we employ a KNN classifier with $k=10$, as it appeared to be the optimal value. Learned features are extracted and used for inference on the test partition of each dataset. Again, we specify MKAN and KAN as [input size, output size, grid range, grid size] and the configurations are listed in the Table \ref{tab:mkans_selfsupervised_config_real}, except the case for 32 feature sizes of MKANs in Fashion MNIST, where we use a grid size 8 to achieve performance higher than KANs with features size 16. Here, we A100 GPU and a single run take at most 10 minutes. 

\begin{table}[h]
    \centering
    \begin{tabular}{|c|c|c|c|}
       Architecture & Encoder config & Decoder config & Parameters \\ \hline  
       MKAN  &  $[N_{in}, N_{out}, (0,1), 5]$ & $[N_{out}, N_{in}, (-3,3), 5]$ & 175K, 175K \\ \hline
       KAN  & $[N_{in}, N_{out}, (0,1), 5]$ & $[N_{out}, N_{in}, (-3,3), 5]$ & 175K, 175K  \\ \hline
    \end{tabular}
    \caption{The table represents encoder and decoder configurations and parameters numbers across MKANs and KANs. Parameters are calculated for the edge case of $N_{in}=784, N_{out}=32$.}
    \label{tab:mkans_selfsupervised_config_real}
\end{table}

Figures~\ref{fig:mkan_latent}, \ref{fig:mkan_latent-fashion} demonstrate monotone latent alignment in a setting where monotone transformations carry semantic meaning: increasing one of the learned latent components monotonically transforms digit ``3'' into digit ``9'' by drawing the missing stroke or transforms t-shirt to pullover.

\begin{figure}[!hh]
\centering
\includegraphics[width=0.95\columnwidth]{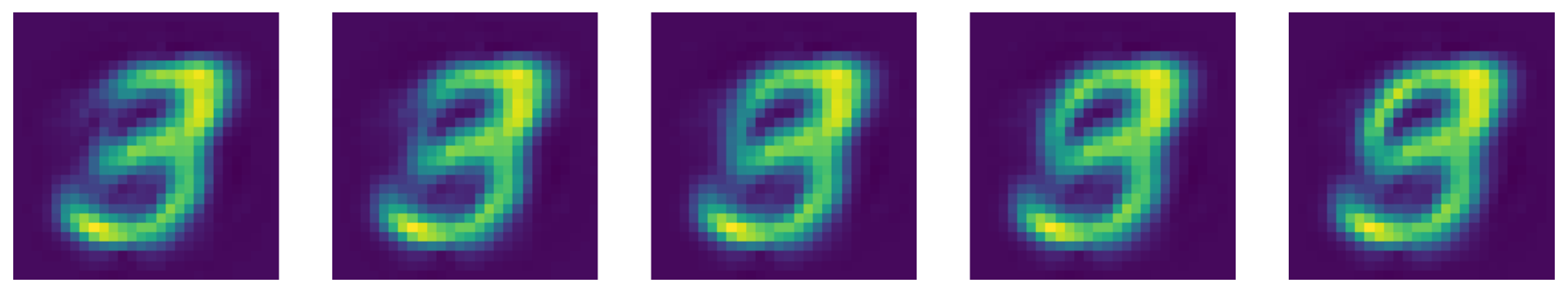}
    \caption{
     Latent traversal in a separate MKAN-VAE trained on Fashion MNIST: the figure shows the evolution of an initial observation (t-shirt) into a pullover under monotone increase of one learned latent component, illustrating that monotone latent alignment generalizes beyond the datasets in Figure~\ref{fig:kan_mkan_acc}.}
    \label{fig:mkan_latent}
\end{figure}

\begin{figure}[!hh]
\centering
\includegraphics[width=0.95\columnwidth]{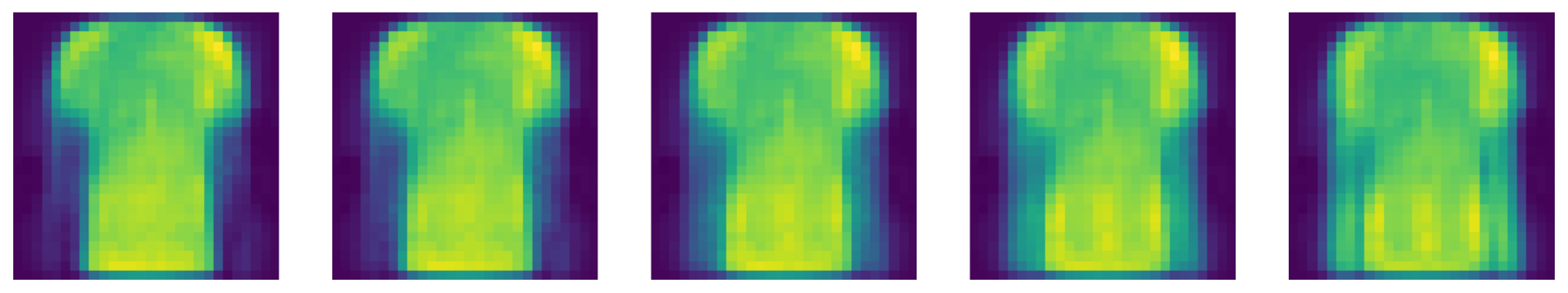}
    \caption{
    The figure demonstrates the evolution of the initial observation, an image of t-shirt, obtained through the MKAN decoder from trained VAE, to an image of pullover, while increasing the value of one of the latent components.}
    \label{fig:mkan_latent-fashion}
\end{figure}
Here, we also provide a short description and motivation for the datasets. We utilize MNIST, Fashion MNIST, Dry Bean~\cite{dry_bean_602}, and 
MAGIC Gamma Telescope~\cite{magic_gamma_telescope_159} datasets from the UCI repository, which were selected to ensure that monotonic feature transformations carry genuine 
semantic meaning, and to cover a range of scales and class cardinalities. In MNIST and Fashion MNIST, 
monotonically increasing pixel intensities in specific spatial regions induce meaningful 
class transitions (e.g., $3 \to 9$, $1 \to 7$ or t-shirt to pullover). In Dry Bean and MAGIC Gamma 
Telescope, features correspond to physical measurements whose monotonic changes admit 
direct domain interpretation - larger area unambiguously implies a larger bean, while 
higher shower size implies a brighter gamma-ray event

\end{document}